\documentclass[11pt]{article}

\newif\ifanonymized
\anonymizedfalse    

\usepackage[margin=1in]{geometry}
\usepackage[T1]{fontenc}
\usepackage[utf8]{inputenc}
\usepackage{lmodern}
\usepackage{microtype}

\usepackage{amsmath,amssymb,amsfonts,amsthm,mathtools}
\usepackage{graphicx}
\usepackage{booktabs}
\usepackage{enumitem}
\usepackage[numbers,sort&compress]{natbib}
\usepackage[hidelinks]{hyperref}
\usepackage[nameinlink,capitalize,noabbrev]{cleveref}

\usepackage{microtype}
\usepackage{hyperref}
\usepackage{url}
\usepackage{booktabs}
\usepackage{tikz}
\usepackage{float}
\usepackage{algorithm}
\usepackage{algpseudocode}
\usepackage{amsmath,amssymb,amsthm,mathtools}
\usepackage{xcolor}
\usepackage{enumitem}
\usepackage{subcaption}
\usepackage{xspace}

\newcommand{\github}{\raisebox{-1.3pt}{\includegraphics[height=1.05em]{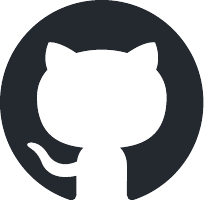}}\xspace}
\newcommand{\worldwideweb}{\raisebox{-1.3pt}{\includegraphics[height=1.05em]{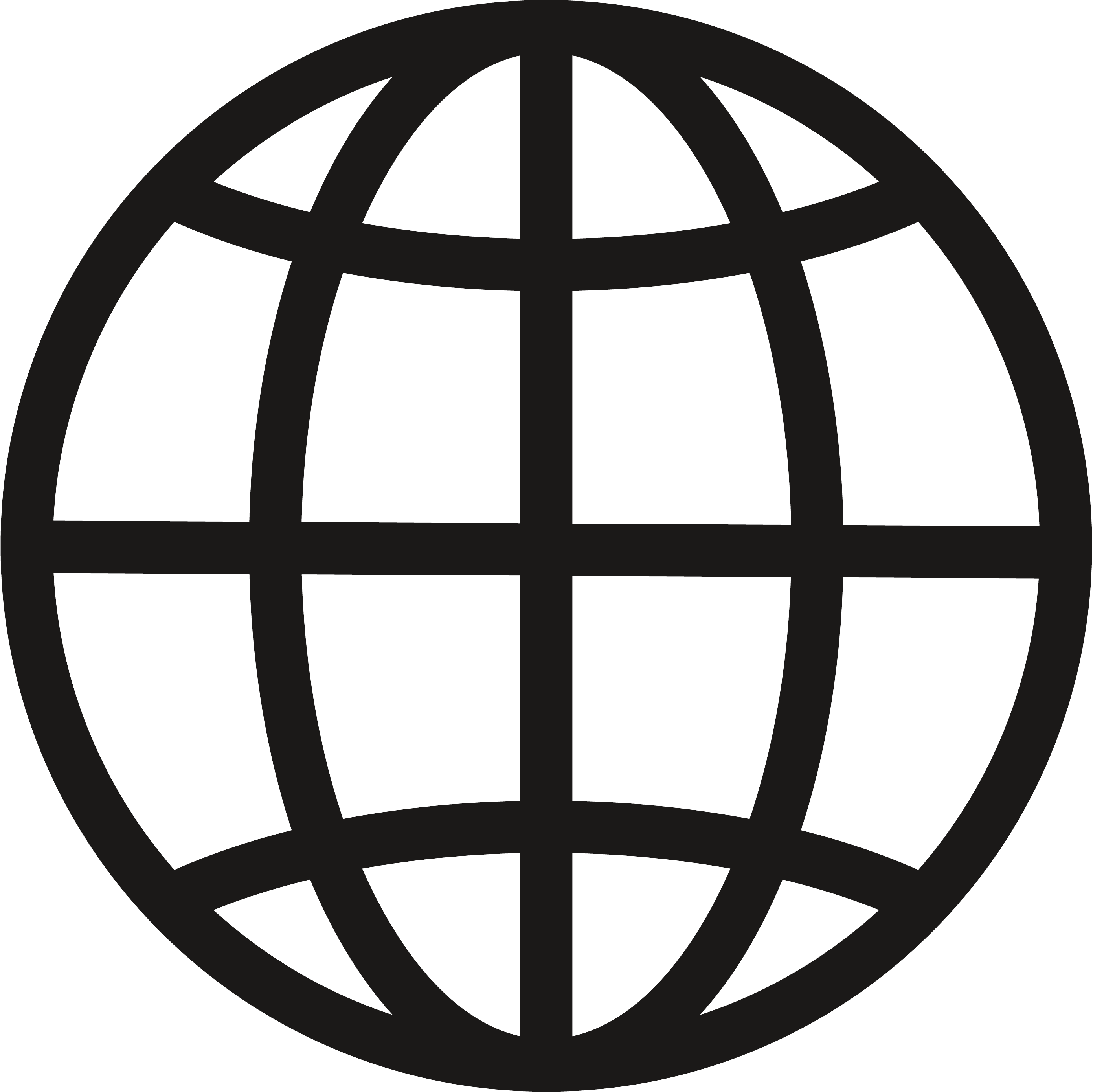}}\xspace}

\usetikzlibrary{arrows.meta}

\newcommand{\E}{\mathbb{E}}
\newcommand{\vocab}{\mathcal{V}}

\newtheorem{lemma}{Lemma}
\newtheorem{proposition}{Proposition}

\theoremstyle{definition}

\theoremstyle{remark}
\newtheorem{remark}{Remark}

\title{Accelerating Speculative Decoding with Block Diffusion Draft Trees}
\date{}

\author{
Liran Ringel$^{1}$ \and Yaniv Romano$^{1,2}$
}

\begin{document}

\maketitle

\begingroup
\renewcommand\thefootnote{}
\footnotetext{
$^{1}$Department of Computer Science, Technion -- Israel Institute of Technology.
$^{2}$Department of Electrical and Computer Engineering, Technion -- Israel Institute of Technology.
Correspondence to: Liran Ringel <\texttt{liranringel@cs.technion.ac.il}>.
}
\endgroup

\vspace{-2.5em}

\begin{abstract}
\noindent Speculative decoding accelerates autoregressive language models by using a lightweight drafter to propose multiple future tokens, which the target model then verifies in parallel. DFlash shows that a block diffusion drafter can generate an entire draft block in a single forward pass and achieve state-of-the-art speculative decoding performance, outperforming strong autoregressive drafters such as EAGLE-3. Vanilla DFlash, however, still verifies only a single drafted trajectory per round, potentially limiting its acceptance length. We introduce DDTree (Diffusion Draft Tree), a method that constructs a draft tree directly from the per-position distributions of a block diffusion drafter. Under a fixed node budget, DDTree uses a simple best-first heap algorithm to select the continuations that are most likely to match the target model according to a surrogate defined by the draft model's output. The resulting tree is verified efficiently in a single target model forward pass using an ancestor-only attention mask. Because DDTree builds on DFlash, a leading draft model for speculative decoding, these gains place DDTree among the leading approaches to speculative decoding.
\end{abstract}

\vspace{-1.4em}

\begin{center}
{\normalfont\fontsize{10}{13}\selectfont
  \href{https://liranringel.github.io/ddtree}{\worldwideweb~Project Page}
  \hspace{0.9cm}
  \href{https://github.com/liranringel/ddtree}{\github~Code}
}
\end{center}

\vspace{-1em}

\begin{figure}[H]
\centering
\includegraphics[width=0.8\linewidth]{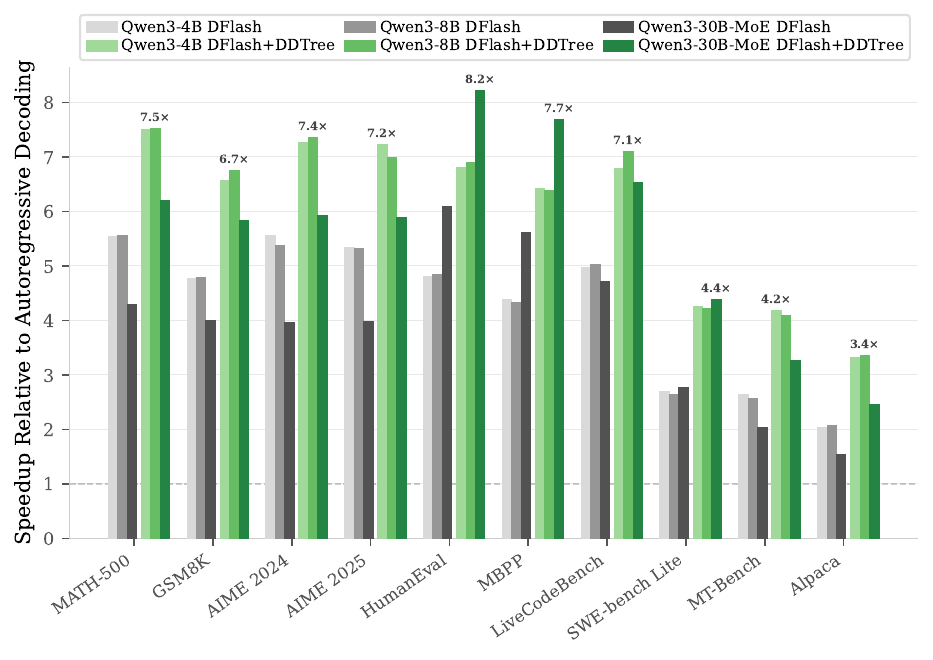}
\caption{Speedups relative to autoregressive decoding at temperature 0.0 across datasets and target models. DDTree bars use the best tree-node budget for each dataset-model pair.}
\label{fig:speedup-overview}
\end{figure}

\section{Introduction}
Autoregressive language models generate text one token at a time, so the sampling of each new token requires another forward pass through a large model. This sequential dependence makes decoding a major source of inference latency. Speculative decoding addresses this bottleneck by using a lightweight drafter model to propose several future tokens and a large target model to verify them in parallel, while preserving the target model's output distribution \citep{leviathan2023fast, chen2023accelerating}. As language models continue to grow in size, reducing decoding latency without changing model outputs has made speculative decoding an increasingly important technique for inference.

The effectiveness of speculative decoding depends on the quality of the draft model. To deliver end-to-end speedups, the drafter must be cheap enough that drafting overhead stays small and accurate enough that the target model frequently accepts the drafted tokens. At a high level, the objective is to maximize the expected number of drafted tokens that the target model accepts, i.e., tokens that agree with the target model and can therefore be committed. At the same time, this should be achieved without adding significant overhead that makes the speedup disappear.

Block diffusion \citep{arriola2025block} is especially attractive in this setting as it can generate an entire draft block in a single forward pass. DFlash \citep{chen2026dflash} shows the promise of this approach: it uses a small block diffusion drafter that leverages features derived from the larger target model, achieving state-of-the-art speculative decoding performance. Indeed, DFlash has been shown to outperform strong autoregressive drafters such as EAGLE-3 \citep{li2025eagle3}. These results establish block diffusion as a powerful foundation for speculative decoding. At the same time, these promising results spark a key challenge:

{\centering \emph{How should we make the best use of the information that the block diffusion draft model produces?}}

Currently, DFlash verifies only one drafted trajectory per round, even though a single block diffusion forward pass produces a distribution over tokens at each future position. As such, DFlash does not utilize the many plausible continuations that block diffusion produces. Naturally, exploring more continuations could increase the probability that the target model continues along a drafted path. On the other hand, naively using multiple continuations increases the verifier cost and can erase the latency benefit.
The challenge, therefore, is to use the draft model's per-position distributions to choose the continuations that are most worthwhile to verify.

We address this challenge with DDTree (Diffusion Draft Tree). Our method (i) constructs a draft tree directly from the per-position distributions produced by a block diffusion drafter, (ii) selects a compact set of promising continuations under a specified tree-node budget, and (iii) verifies them in a single target model forward pass with tree attention. In this way, DDTree preserves DFlash's low drafting latency while allowing each round to explore multiple continuations instead of only one. Our experiments show that DDTree consistently improves over vanilla DFlash across models, and Figure~\ref{fig:speedup-overview} summarizes these gains.

\paragraph{Contributions.}
\begin{itemize}[leftmargin=1.5em]
    \item We introduce DDTree, a speculative decoding method that constructs a draft tree directly from the per-position distributions produced by a single block diffusion forward pass within a fixed node budget.
    \item We show that the tree we construct provably maximizes the expected number of accepted tokens under the draft model, which we use as a surrogate for the target model's expected acceptance length. We also show that the optimal tree can be recovered with a simple and efficient best-first heap algorithm. This result builds on OPT-Tree~\citep{wang2025opt}, but adapts it to the block diffusion setting, where the tree is constructed from the per-position distributions of a single forward pass rather than the multiple forward passes required by autoregressive drafters.
    \item We implement DDTree on top of the Hugging Face Transformers library and show consistent gains over vanilla DFlash across model sizes and domains.
\end{itemize}

\section{Related Work}

\textbf{Speculative decoding and tree-based verification.}
Foundational speculative decoding methods verify a drafted continuation against the target model in parallel and preserve the target model's output distribution \citep{leviathan2023fast, chen2023accelerating}. A subsequent line of work generalizes this setting from a single drafted continuation to a tree of candidate continuations \citep{spector2023accelerating, jeon2024recursive, xiong2025dyspec}. SpecInfer by~\citet{miao2023specinfer} introduces tree attention for efficient target-model verification over token trees, and Medusa~\citep{cai2024medusa} combines multiple prediction heads with the same style of tree-based verification. The EAGLE family extends this direction using target-model features: EAGLE drafts in feature space \citep{li2024eagle}, EAGLE-2 adds dynamic draft-tree construction \citep{li2024eagle2}, and EAGLE-3 predicts tokens directly from fused multi-layer features \citep{li2025eagle3}. OPT-Tree by~\citet{wang2025opt} constructs adaptive trees for autoregressive drafters by maximizing an approximate expected acceptance length under a node budget. In that autoregressive setting, tree construction still requires one drafter's forward pass per tree depth (i.e., per token position), and thus it has higher computational overhead compared to our DDTree approach.

\textbf{Parallel drafting and block diffusion drafters.}
A complementary line of work reduces drafting latency by predicting multiple future tokens in a single forward pass. Block Diffusion provides the modeling foundation for prefix-conditioned blockwise denoising with KV-cache-friendly generation~\citep{arriola2025block}. PARD~\citep{an2026pard} studies low-cost parallel drafting by adapting autoregressive models to mimic diffusion-style block prediction. DFlash~\citep{chen2026dflash} shows that a small block diffusion drafter, conditioned on target-model features, can predict an entire block in one forward pass and then verify a single drafted trajectory losslessly.

Vanilla DFlash, however, explores only one continuation per round. A very recent work, DART~\citep{liu2026dart}, also constructs draft trees from one-pass parallel logits. Still, it relies on continuity-aware tree pruning with an external $N$-gram continuity score and a large $N$-gram trie at runtime~\citep{liu2026dart}. By contrast, our proposed DDTree keeps the same one-pass DFlash drafter and constructs the tree directly from the per-position probabilities produced by that pass. This avoids auxiliary external scoring and gives an explicit surrogate objective that our best-first construction provably maximizes.

\section{Background: Block Diffusion Drafting}

Speculative decoding methods proceed in rounds. At the beginning of each round, we already have one new token produced by the target model; we denote this `bonus' token by $b$. This token is guaranteed to exist, even if no drafted tokens were accepted in the previous round. Importantly, although $b$ has been selected by the target model, the target model has not yet been run on $b$, i.e., no forward pass has been performed with $b$ appended to the context. As a result, the drafter can condition on the identity of 
$b$, but any target-model features it uses, as in DFlash, are only available for the preceding context.

With \(b\) fixed, the drafter's task is to predict the next tokens after \(b\). A block diffusion drafter does this in parallel, producing a short block of future tokens rather than generating them autoregressively one token at a time.
Given the context and \(b\), the drafter takes a masked block of the form \([b, m, \ldots, m]\) and predicts tokens for the next \(L\) masked positions in one forward pass. This produces logits
\[
\ell_i \in \mathbb{R}^{|\vocab|}, \qquad i=1,\dots,L,
\]
or, equivalently, per-position token distributions
\[
q_i(v)=\operatorname{softmax}(\ell_i)_v, \qquad v\in\vocab.
\]
Crucially, each \(q_i\) is a \emph{marginal} distribution for position \(i\), not a path-conditioned distribution. Because the block is predicted in parallel, the prediction at position \(i\) conditions on the context before the drafted block (and, in DFlash, on target-model features), but not on the specific token choices at positions \(1,\dots,i-1\) inside that same block. Thus, a single DFlash pass provides one marginal distribution per future position, rather than conditional probabilities for every partial continuation.

This distinction is central to DDTree. Let \(c\) denote the current context, let \(b\) denote the bonus token available at the start of the round, and let \(y_{1:L}\) denote a candidate continuation after \(b\). Under the target model, the continuation distribution factorizes autoregressively as
\begin{equation}
    p(y_{1:L} \mid c,b)
    =
    \prod_{i=1}^{L} p(y_i \mid c,b,y_{1:i-1}).
\label{eq:target-factorization}
\end{equation}
A one-pass block diffusion drafter does not expose these continuation-conditioned factors. Instead, it provides only per-position marginals \(\{q_i\}_{i=1}^L\), where each \(q_i(\cdot \mid c,b)\) predicts the token at position \(i\) from the shared conditioning context \((c,b)\), without conditioning on the realized tokens at earlier positions within the same block. Accordingly, the natural distribution associated with a one-pass block diffusion drafter is the factorized distribution
\begin{equation}
    Q(y_{1:L}\mid c,b)
    :=
    \prod_{i=1}^{L} q_i(y_i\mid c,b).
\label{eq:factorized-approx}
\end{equation}
Thus, the target model provides a path-conditioned autoregressive distribution, whereas the drafter provides only a factorized distribution over the next \(L\) positions.

Our proposed DDTree builds on the factorized distribution~\eqref{eq:factorized-approx}: we use it to define the surrogate objective for selecting a draft tree from a single block diffusion drafter forward pass. As in standard speculative decoding with a block diffusion drafter, DDTree uses one lightweight drafter pass followed by target-model verification. The key difference is how we use the resulting one-pass marginals \(\{q_i\}_{i=1}^L\). Rather than collapsing the marginals into a single continuation, we use them to construct a compact draft tree for verification.

\section{The Proposed Method: Diffusion Draft Tree}
\subsection{Overview}
DDTree builds a draft tree under a node budget $B$, in which a node at depth $i$ represents a candidate token for the $i$-th future position.
Ideally, in each decoding round, we would choose a draft tree with $B$ nodes that maximizes the expected number of speculative tokens accepted by the target model.
A single block diffusion drafter forward pass outputs, in parallel, one drafter-predicted marginal token distribution for each future position in the next block. It does not provide the target-conditioned continuation probabilities in~\eqref{eq:target-factorization} needed to optimize the expected target-model acceptance length directly.
We therefore can only optimize a surrogate objective, which is the expected acceptance length under the drafter's factorized approximation~\eqref{eq:factorized-approx}. In Section~\ref{sec:surrogate-objective}, we formalize this surrogate and derive the corresponding tree-construction objective.

At each decoding round, let $c$ denote the full current context, including the prompt and all previously generated tokens. Recall that the bonus token $b$ is already obtained by the target model: in the first round it comes from the prefill pass, and, in later rounds it is the previous round's bonus token. With this in place, conditioned on $c$ and $b$, our proposed DDTree method implements speculative decoding for the next block by applying four steps:
\begin{figure}[t]
\centering

\begin{minipage}[t]{0.49\linewidth}
\centering
\begin{tikzpicture}[
    x=0.74cm,
    y=0.74cm,
    >=Latex,
    font=\scriptsize,
    rootnode/.style={circle, draw=blue!65!black, fill=blue!12, line width=0.9pt, minimum size=5.1mm, inner sep=0pt},
    draftnode/.style={circle, draw=black!65, fill=gray!10, minimum size=4.7mm, inner sep=0pt},
    contextbox/.style={draw=black!35, fill=gray!12, rounded corners=2pt, minimum width=1.45cm, minimum height=6mm}
]
    \foreach \x/\d/\dx in {1.45/1/-0.34,3.05/2/0,4.65/3/0.34} {
        \draw[gray!18] (\x,-1.9) -- (\x,1.95);
        \node[font=\scriptsize, text=black!55] at ({\x+\dx},2.2) {position \d};
    }

    \node[contextbox] (ctx) at (-2.05,0) {context};

    \node[rootnode] (r) at (0,0) {$a$};
    \node[font=\scriptsize, text=blue!70!black, align=center] at (0,-0.79) {previous\\bonus};
    \draw[-Latex, draw=black!45, line width=0.8pt] (ctx.east) -- (r.west);

    \node[draftnode] (a) at (1.45,0.7) {$b$};
    \node[draftnode] (b) at (1.45,-0.85) {$c$};
    \node[draftnode] (c) at (3.05,1.18) {$d$};
    \node[draftnode] (d) at (3.05,0.08) {$e$};
    \node[draftnode] (e) at (3.05,-1.35) {$f$};
    \node[draftnode] (f) at (4.65,1.42) {$g$};
    \node[draftnode] (g) at (4.65,0.65) {$h$};
    \node[draftnode] (h) at (4.65,-0.22) {$i$};
    \draw[draw=black!35, line width=0.7pt] (r) -- (a);
    \draw[draw=black!35, line width=0.7pt] (r) -- (b);
    \draw[draw=black!35, line width=0.7pt] (a) -- (c);
    \draw[draw=black!35, line width=0.7pt] (a) -- (d);
    \draw[draw=black!35, line width=0.7pt] (b) -- (e);
    \draw[draw=black!35, line width=0.7pt] (c) -- (f);
    \draw[draw=black!35, line width=0.7pt] (c) -- (g);
    \draw[draw=black!35, line width=0.7pt] (d) -- (h);
\end{tikzpicture}

\vspace{0.25em}
\parbox{0.94\linewidth}{\centering \textbf{(a)} Tree selected from a single block diffusion drafter forward pass}
\end{minipage}
\hfill
\begin{minipage}[t]{0.49\linewidth}
\centering
\begin{tikzpicture}[
    x=0.74cm,
    y=0.74cm,
    >=Latex,
    font=\scriptsize,
    rootnode/.style={circle, draw=blue!65!black, fill=blue!12, line width=0.9pt, minimum size=5.1mm, inner sep=0pt},
    draftnode/.style={circle, draw=black!45, fill=gray!7, minimum size=4.7mm, inner sep=0pt},
    acceptnode/.style={circle, draw=green!45!black, fill=green!18, line width=1.0pt, minimum size=4.9mm, inner sep=0pt},
    bonusnode/.style={circle, draw=orange!80!black, fill=orange!20, line width=1.0pt, minimum size=5.1mm, inner sep=0pt},
    contextbox/.style={draw=black!35, fill=gray!12, rounded corners=2pt, minimum width=1.45cm, minimum height=6mm}
]
    \foreach \x/\d/\dx in {1.45/1/-0.34,3.05/2/0,4.65/3/0.34} {
        \draw[gray!18] (\x,-1.9) -- (\x,1.95);
        \node[font=\scriptsize, text=black!55] at ({\x+\dx},2.2) {position \d};
    }

    \node[contextbox] (ctx) at (-2.05,0) {context};

    \node[rootnode] (r) at (0,0) {$a$};
    \node[font=\scriptsize, text=blue!70!black, align=center] at (0,-0.79) {previous\\bonus};
    \draw[-Latex, draw=black!45, line width=0.8pt] (ctx.east) -- (r.west);

    \node[acceptnode] (a) at (1.45,0.7) {$b$};
    \node[draftnode] (b) at (1.45,-0.85) {$c$};
    \node[draftnode] (c) at (3.05,1.18) {$d$};
    \node[acceptnode] (d) at (3.05,0.08) {$e$};
    \node[draftnode] (e) at (3.05,-1.35) {$f$};
    \node[draftnode] (f) at (4.65,1.42) {$g$};
    \node[draftnode] (g) at (4.65,0.65) {$h$};
    \node[draftnode] (h) at (4.65,-0.22) {$i$};
    \node[bonusnode] (bonus) at (5.55,-1.05) {$a'$};
    \node[font=\scriptsize, text=orange!90!black, align=center] at (5.55,-1.82) {next\\bonus};

    \draw[draw=black!25, line width=0.6pt] (r) -- (a);
    \draw[draw=black!25, line width=0.6pt] (r) -- (b);
    \draw[draw=black!25, line width=0.6pt] (a) -- (c);
    \draw[draw=black!25, line width=0.6pt] (a) -- (d);
    \draw[draw=black!25, line width=0.6pt] (b) -- (e);
    \draw[draw=black!25, line width=0.6pt] (c) -- (f);
    \draw[draw=black!25, line width=0.6pt] (c) -- (g);
    \draw[draw=black!25, line width=0.6pt] (d) -- (h);

    \draw[-Latex, draw=green!45!black, line width=1.7pt] (r) -- (a);
    \draw[-Latex, draw=green!45!black, line width=1.7pt] (a) -- (d);
    \draw[-Latex, dashed, draw=orange!80!black, line width=1.2pt] (d) -- (bonus);
\end{tikzpicture}

\vspace{0.25em}
\parbox{0.94\linewidth}{\centering \textbf{(b)} Verifier walk: two matches, then the next bonus token}
\end{minipage}

\caption{Illustration of one DDTree decoding round. The bonus token $a$ from the previous round is carried into the current round and serves as the tree root. One block diffusion drafter pass produces per-position marginals for the next three draft positions, from which DDTree selects a draft tree. During verification, the target model follows its own decoding rule, and the walk continues whenever the selected token appears as a child in the tree. In this example, two speculative nodes are matched (green), the matched path is added to the output, and the first unmatched target token $a'$ becomes the bonus token for the next round.}
\label{fig:round-overview}
\end{figure}
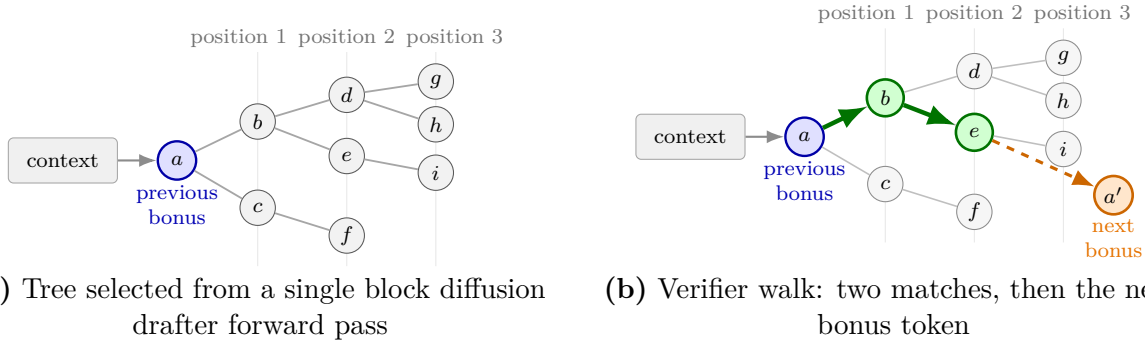

\begin{enumerate}[leftmargin=1.5em]
    \item Run the block diffusion drafter once to obtain per-position distributions for the next $L$ draft positions after $b$.
    \item Build a draft tree with $B$ nodes from those distributions. The node budget $B$ controls how many candidate tokens the verifier (target model) evaluates.
    \item Compile the tree into input tensors for the target model and run one target-model forward pass with tree attention.
    \item Walk the tree as illustrated in Figure~\ref{fig:round-overview}: at each step, use the target model's decoding rule to choose the next token and check whether it matches a child in the tree; accept the matched drafted path, and carry the first unmatched target token (the next bonus token) to the next round.
\end{enumerate}

\subsection{Surrogate objective for draft-tree selection}
\label{sec:surrogate-objective}

Consider a decoding round with current context \(c\) and bonus token \(b\). A single block diffusion drafter forward pass produces \(L\) per-position distributions \(q_i(\cdot)\), \(i=1,\dots,L\), for the next \(L\) positions after \(b\). For \(1\le d\le L\), let
\begin{equation}
    u=(u_1,\dots,u_d)\in\vocab^d
\end{equation}
denote a candidate continuation prefix after \(b\). We identify each tree node with such a prefix, so a node at depth \(d\) represents the path \(u_1,\dots,u_d\) from the root \(b\).\footnote{Throughout the paper, we use ``node'' to refer either to the full prefix \(u\) or, when the context is unambiguous, to its final token \(u_d\).}

A draft tree \(T\) is a set of such prefixes. The tree is valid if it is prefix-closed: whenever \((u_1,\dots,u_d)\in T\) with \(d\ge 2\), the parent prefix \((u_1,\dots,u_{d-1})\) also lies in \(T\). In other words, if a node is included in the tree, then so are all of its ancestors back to the root \(b\). The root \(b\) itself does not count toward the node budget.
Let
\[
N = \sum_{d=1}^{L} |\vocab|^d
\]
be the total number of nonempty prefixes of length at most \(L\). In the regime relevant for DDTree, this number is enormous, and we always choose node budgets with \(B \ll N\). Accordingly, throughout the rest of this section we treat \(B < N\) as a standing assumption.

For a candidate continuation \(y_{1:L}\), we define its acceptance length under \(T\) as the longest prefix of \(y_{1:L}\) that appears in the tree:
\begin{equation}
    \alpha_T(y_{1:L}) = \max\{d : y_{1:d} \in T\},
\end{equation}
with \(\alpha_T(y_{1:L}) = 0\) when no depth-1 node matches. With this in place, the ideal tree-construction objective is
\begin{equation}
    \max_{T:\,|T|\le B,\;T\text{ valid}} \E_{Y_{1:L}\sim p(\cdot \mid c,b)}[\alpha_T(Y_{1:L})],
\end{equation}
namely, the expected acceptance length under the target model \(p\).
This objective, however, is infeasible to optimize as it depends on the target-conditioned continuation probabilities~\eqref{eq:target-factorization}, which are unavailable when the tree is constructed.

We therefore use a surrogate based on the factorized distribution induced by the block diffusion drafter $Q(\cdot \mid c,b)$, constructed solely from the per-position marginals \(\{q_i\}_{i=1}^L\) in~\eqref{eq:factorized-approx}. This surrogate does not require full path-conditioned probabilities, which are unavailable from a single block diffusion forward pass.

Accordingly, we choose the draft tree by maximizing the expected acceptance length under the factorized distribution:
\begin{equation}
\label{eq:surrogate-objective}
    \max_{T:\,|T|\le B,\;T\text{ valid}} \E_{Y_{1:L}\sim Q(\cdot \mid c,b)}[\alpha_T(Y_{1:L})].
\end{equation}
To analyze this surrogate objective, we express it in terms of prefix probabilities under \(Q\). For a prefix \(u = (u_1, \dots, u_d)\), the probability that a sampled continuation under \(Q\) begins with \(u\) is
\begin{equation}
    q(u \mid c,b) = \prod_{i=1}^{|u|} q_i(u_i \mid c,b).
    \label{eq:prefix-mass}
\end{equation}
The above is the key quantity for tree construction, because the acceptance length is determined by how far the sampled continuation matches a prefix in the tree. The next proposition shows that the surrogate objective decomposes as an additive sum of prefix masses over the nodes in \(T\).

The above is the key quantity for tree construction, because the acceptance length is determined by the longest prefix of the sampled continuation that matches a path in the tree. The next proposition shows that the surrogate objective decomposes as an additive sum of prefix masses over the nodes in \(T\).

\begin{proposition}
\label{prop:surrogate-additive}
For any valid draft tree \(T\),
\begin{equation}
    \E_{Y_{1:L} \sim Q(\cdot \mid c,b)}[\alpha_T(Y_{1:L})] = \sum_{u \in T} q(u \mid c,b).
\end{equation}
\end{proposition}
\noindent All proofs are in Appendix~\ref{app:proofs}. Proposition~\ref{prop:surrogate-additive} shows that the surrogate objective is an additive sum of prefix probabilities
$\sum_{u \in T} q(u \mid c,b).$ Therefore, selecting a draft tree reduces to choosing a valid set of at most \(B\) prefixes that maximizes this sum. Since the objective is additive over nodes and all terms are nonnegative, the optimal solution is to include the highest-probability prefixes up to the budget \(B\), subject to the prefix-closure constraint.
The next proposition shows that these top-\(B\) prefixes automatically form a valid tree and are therefore optimal.

For the next proposition and the results that follow, we assume \(q_i(v \mid c,b) \in (0,1)\) for all depths \(i\) and tokens \(v\). This is the standard case when the draft model outputs probabilities via a softmax over finite logits. If needed, the construction can be modified slightly to handle edge cases without changing the main results.

\begin{proposition}
\label{prop:top-b-tree}
Let \(u^{(1)}, u^{(2)}, \dots\) be all nonempty prefixes of length at most \(L\), ordered so that
\[
q\bigl(u^{(1)} \mid c,b\bigr) \ge q\bigl(u^{(2)} \mid c,b\bigr) \ge \dots,
\]
with ties broken arbitrarily. Define
\[
T_B = \{u^{(1)}, \dots, u^{(B)}\}.
\]
Then, \(T_B\) is a valid draft tree with \(|T_B| \le B\). Moreover, \(T_B\) maximizes \(\E_{Y_{1:L} \sim Q(\cdot \mid c,b)}[\alpha_T(Y_{1:L})]\) among all valid draft trees with \(|T| \le B\).
\end{proposition}

The above result is analogous to OPT-Tree's expected acceptance length objective \citep{wang2025opt}, with the crucial difference that in our case all required probabilities are obtained from a single block diffusion drafter forward pass rather than from multiple autoregressive passes.

\begin{remark}
The result stated in Proposition~\ref{prop:top-b-tree} is exact for the surrogate objective induced by the factorized draft distribution \(Q(\cdot \mid c,b)\), not for the true target-model distribution. The task here is not to recover the unavailable target model path-conditioned probabilities, but to make the best use of the information available from one block diffusion drafter forward pass.
\end{remark}

\subsection{Efficient and optimal tree construction}
\label{sec:best-first-construction}

Proposition~\ref{prop:top-b-tree} reveals that an optimal draft tree is obtained by taking the $B$ highest-probability prefixes. The remaining challenge is therefore algorithmic: how can we recover these prefixes efficiently, without enumerating the exponentially many possible prefixes up to depth \(L\)? In what follows, we show that this can be done with a simple best-first search procedure.

At each depth \(i\), let \(v_i^{(1)}, v_i^{(2)}, \dots\) be the tokens ordered so that
\[
q_i(v_i^{(1)} \mid c,b) \ge q_i(v_i^{(2)} \mid c,b) \ge \dots,
\]
and let \(q_i^{(k)} = q_i(v_i^{(k)} \mid c,b)\). We index prefixes by token ranks rather than by vocabulary ids. Specifically, a rank tuple
\[
\rho=(\rho_1,\dots,\rho_d),
\qquad
1 \le d \le L,
\qquad
1 \le \rho_i \le |\vocab|,
\]
denotes the depth-\(d\) prefix $(v_1^{(\rho_1)},\dots,v_d^{(\rho_d)})$.
For example, \(\rho=(1,3,2)\) denotes the prefix obtained by taking the most probable token at position 1, the third-most probable token at position 2, and the second-most probable token at position 3.
The probability of a prefix \(\rho\) is
\[
q(\rho)=\prod_{i=1}^d q_i^{(\rho_i)},
\]
and its log-probability is
\[
\sigma(\rho)=\log q(\rho)=\sum_{i=1}^d \log q_i^{(\rho_i)}.
\]
We will use \(\sigma(\rho)\) to rank prefixes in a max-heap, which we use as a priority queue that returns the prefix with the largest value first. Taking logarithms converts products into sums, improving numerical stability while preserving the ordering of prefixes.

Let \(K = \min(B,|\vocab|)\), and let
\[
\mathcal{S}_K
=
\{(\rho_1,\dots,\rho_d): 1 \le d \le L,\; 1 \le \rho_i \le K \text{ for } i=1,\dots,d\}.
\]
Thus, \(\mathcal{S}_K\) contains exactly the prefixes that use only the top-\(K\) tokens at every depth.

\begin{lemma}
\label{lem:top-k-reduction}
There exists an optimal valid draft tree that maximizes the surrogate expected acceptance length in \eqref{eq:surrogate-objective}, such that every node in the tree lies in \(\mathcal{S}_K\). Equivalently, there exists an optimal valid draft tree in which every node uses only the top-\(K\) tokens at each depth.
\end{lemma}

Lemma~\ref{lem:top-k-reduction} reduces the search space to the top-\(K\) tokens at each depth. Algorithm~\ref{alg:tree-build} enumerates the retained prefixes \(\mathcal{S}_K\) in descending order of \(\sigma(\rho)\) using a max-heap. It starts from the tuple \((1)\), i.e., the length-1 tuple containing only rank \(1\). When a tuple \(\rho=(\rho_1,\dots,\rho_d)\) is popped, the algorithm generates at most two new tuples: its next sibling, which changes only the last rank to \(\rho_d+1\), and its first child, which appends rank \(1\) at the next depth. Intuitively, the sibling explores an alternative token at the current position, while the child extends the current prefix with the best available token at the next position. After \(B\) pops, the algorithm returns the top-\(B\) prefixes.
The next proposition establishes the optimality of Algorithm~\ref{alg:tree-build}.

\begin{algorithm}[H]
\caption{Best-first draft-tree construction from one block diffusion drafter pass}
\label{alg:tree-build}
\begin{algorithmic}[1]
\Require Top-\(K\) tokens \(\{v_i^{(k)}\}_{i=1,k=1}^{L,K}\) and their probabilities \(\{q_i^{(k)}\}_{i=1,k=1}^{L,K}\); node budget \(B\)
\State Initialize max-heap \(H \gets \{((1), \sigma((1)))\}\)
\State Initialize draft tree \(T \gets \varnothing\)
\While{\(|T| < B\) and \(H \neq \varnothing\)}
    \State Pop the rank tuple \(\rho=(\rho_1,\dots,\rho_d)\) with largest score \(\sigma(\rho)\)
    \State Add prefix \((v_1^{(\rho_1)},\dots,v_d^{(\rho_d)})\) to \(T\)
    \If{\(\rho_d + 1 \le K\)}
        \State Push next sibling \((\rho_1,\dots,\rho_{d-1},\rho_d+1)\) with score \(\sigma(\rho) - \log q_d^{(\rho_d)} + \log q_d^{(\rho_d+1)}\)
    \EndIf
    \If{\(d < L\)}
        \State Push first child \((\rho_1,\dots,\rho_d,1)\) with score \(\sigma(\rho) + \log q_{d+1}^{(1)}\)
    \EndIf
\EndWhile
\State \Return draft tree \(T\)
\end{algorithmic}
\end{algorithm}

\begin{proposition}
\label{prop:alg_optimal_tree}
Algorithm~\ref{alg:tree-build} returns an optimal valid draft tree for the surrogate objective in \eqref{eq:surrogate-objective} under node budget \(B\).
\end{proposition}

\begin{remark}
    The heap stage costs \(O(B \log B)\): the algorithm performs at most \(B\) pops and at most \(2B\) pushes, and the heap size is \(O(B)\) throughout.
\end{remark}

\subsection{Efficient verification and cache update}
\label{sec:verify}

Armed with an algorithm that efficiently constructs the optimal draft tree under the node budget, we now describe how the selected draft tree is compiled for verification, how the verifier walk is performed, and how the KV cache is updated afterward.

To verify the selected draft tree in one target-model forward pass, we flatten it into a sequence of token ids rooted at the bonus token $b$. We assign position ids by tree depth so that the verifier applies the correct positional encoding. We then use tree attention \citep{miao2023specinfer}, under which each drafted node attends to the past context through the KV cache and, within the drafted tree, only to the root, its ancestors, and itself. Together, these inputs let the verifier score all selected tree-branches in a single forward pass.

Verification then follows the target model's own decoding rule, whether greedy or temperature-based sampling. Figure~\ref{fig:round-overview}(b) illustrates the process. Starting from the bonus token $b$, we check whether the token selected by the target model at the current node matches one of that node's children in the draft tree. If it does, that child is accepted and the walk continues. Since the verifier has already produced outputs for all drafted nodes in the same forward pass, this continuation requires no additional target-model call. If no child matches, the walk stops. The accepted path is appended to the output sequence, the first unmatched target-model token becomes the bonus token for the next round, and the KV cache is compacted to retain only the accepted path.

\section{Experiments}

We evaluate whether DDTree improves end-to-end speculative decoding over vanilla DFlash. Our experiments focus on decoding speed and acceptance behavior across target model size, domain, and decoding temperature.

\subsection{Experimental setup}
We evaluate three target models, Qwen3-4B, Qwen3-8B, and Qwen3-Coder-30B-A3B-Instruct~\citep{yang2025qwen3}, each paired with its corresponding DFlash checkpoint available at \url{https://huggingface.co/collections/z-lab/dflash}.
Our benchmark suite spans reasoning tasks such as MATH-500~\citep{lightman2023let}, GSM8K~\citep{cobbe2021training}, AIME 2024, and AIME 2025; code tasks such as HumanEval~\citep{chen2021evaluating}, MBPP~\citep{austin2021program}, LiveCodeBench~\citep{jain2025livecodebench}, and SWE-bench Lite~\citep{jimenez2024swebench}; and general instruction and dialogue tasks such as MT-Bench~\citep{zheng2023judging} and Alpaca~\citep{taori2023alpaca}. We run the benchmark on 8 H200 GPUs at temperatures 0.0 and 1.0 and report speedup relative to autoregressive decoding, mean acceptance length \(\tau\) including the bonus token, and, for the case study below, the acceptance length histogram. Additional benchmark details, including per-dataset sample counts, decoding hyperparameters, warmup, and backend selection, are provided in Appendix~\ref{app:benchmark-details}.

\subsection{Main results}

\begin{table*}[t]
\centering
\small
\caption{Speedup over autoregressive decoding and mean acceptance length (\(\tau\)). Results are reported separately for temperature 0.0 and temperature 1.0. For each dataset, model, and temperature, the DDTree entry uses the best node budget from \(\{16, 32, 64, 128, 256, 512, 1024\}\).}
\label{tab:benchmark-results}
\resizebox{\textwidth}{!}{
\begin{tabular}{l rcrc rcrc rcrc}
\toprule
 & \multicolumn{4}{c}{\textbf{Qwen3-4B}} & \multicolumn{4}{c}{\textbf{Qwen3-8B}} & \multicolumn{4}{c}{\textbf{Qwen3-Coder-30B-A3B-Instruct}} \\
\cmidrule(lr){2-5} \cmidrule(lr){6-9} \cmidrule(lr){10-13}
\textbf{Dataset} & \multicolumn{2}{c}{DFlash} & \multicolumn{2}{c}{DFlash+DDTree} & \multicolumn{2}{c}{DFlash} & \multicolumn{2}{c}{DFlash+DDTree} & \multicolumn{2}{c}{DFlash} & \multicolumn{2}{c}{DFlash+DDTree} \\
\cmidrule(lr){2-3} \cmidrule(lr){4-5} \cmidrule(lr){6-7} \cmidrule(lr){8-9} \cmidrule(lr){10-11} \cmidrule(lr){12-13}
 & Speedup & $\tau$ & Speedup & $\tau$ & Speedup & $\tau$ & Speedup & $\tau$ & Speedup & $\tau$ & Speedup & $\tau$ \\
\midrule
\multicolumn{13}{l}{\textit{Temperature = 0.0}} \\
AIME 2024 & 5.56$\times$ & 7.54 & \textbf{7.27$\times$} & \textbf{10.37} & 5.38$\times$ & 7.46 & \textbf{7.35$\times$} & \textbf{10.42} & 3.95$\times$ & 5.16 & \textbf{5.93$\times$} & \textbf{7.87} \\
AIME 2025 & 5.33$\times$ & 7.37 & \textbf{7.23$\times$} & \textbf{10.23} & 5.32$\times$ & 7.39 & \textbf{6.99$\times$} & \textbf{9.86} & 3.98$\times$ & 5.13 & \textbf{5.88$\times$} & \textbf{7.63} \\
Alpaca & 2.03$\times$ & 3.11 & \textbf{3.32$\times$} & \textbf{5.35} & 2.07$\times$ & 3.12 & \textbf{3.36$\times$} & \textbf{5.09} & 1.53$\times$ & 2.22 & \textbf{2.46$\times$} & \textbf{3.55} \\
GSM8K & 4.77$\times$ & 6.51 & \textbf{6.58$\times$} & \textbf{9.33} & 4.78$\times$ & 6.57 & \textbf{6.75$\times$} & \textbf{9.54} & 4.00$\times$ & 5.18 & \textbf{5.83$\times$} & \textbf{7.55} \\
HumanEval & 4.81$\times$ & 6.62 & \textbf{6.81$\times$} & \textbf{9.44} & 4.84$\times$ & 6.61 & \textbf{6.90$\times$} & \textbf{9.67} & 6.09$\times$ & 8.02 & \textbf{8.22$\times$} & \textbf{10.72} \\
LiveCodeBench & 4.97$\times$ & 7.02 & \textbf{6.78$\times$} & \textbf{9.79} & 5.02$\times$ & 7.22 & \textbf{7.10$\times$} & \textbf{10.28} & 4.72$\times$ & 6.32 & \textbf{6.54$\times$} & \textbf{8.63} \\
MATH-500 & 5.54$\times$ & 7.72 & \textbf{7.50$\times$} & \textbf{10.71} & 5.56$\times$ & 7.79 & \textbf{7.52$\times$} & \textbf{10.73} & 4.29$\times$ & 5.58 & \textbf{6.21$\times$} & \textbf{8.10} \\
MBPP & 4.38$\times$ & 6.10 & \textbf{6.42$\times$} & \textbf{9.16} & 4.33$\times$ & 5.99 & \textbf{6.39$\times$} & \textbf{9.07} & 5.61$\times$ & 7.19 & \textbf{7.68$\times$} & \textbf{9.94} \\
MT-Bench & 2.64$\times$ & 4.38 & \textbf{4.18$\times$} & \textbf{6.70} & 2.56$\times$ & 4.28 & \textbf{4.10$\times$} & \textbf{6.58} & 2.04$\times$ & 3.52 & \textbf{3.27$\times$} & \textbf{5.36} \\
SWE-bench Lite & 2.70$\times$ & 3.66 & \textbf{4.25$\times$} & \textbf{5.99} & 2.65$\times$ & 3.60 & \textbf{4.23$\times$} & \textbf{5.91} & 2.77$\times$ & 3.61 & \textbf{4.38$\times$} & \textbf{5.71} \\
\midrule
\multicolumn{13}{l}{\textit{Temperature = 1.0}} \\
AIME 2024 & 3.50$\times$ & 5.01 & \textbf{5.31$\times$} & \textbf{7.82} & 3.46$\times$ & 4.98 & \textbf{5.36$\times$} & \textbf{7.79} & 3.12$\times$ & 4.23 & \textbf{4.94$\times$} & \textbf{7.32} \\
AIME 2025 & 3.38$\times$ & 4.79 & \textbf{5.08$\times$} & \textbf{7.22} & 3.36$\times$ & 4.79 & \textbf{5.25$\times$} & \textbf{7.71} & 3.17$\times$ & 4.26 & \textbf{4.98$\times$} & \textbf{6.62} \\
Alpaca & 1.97$\times$ & 2.96 & \textbf{3.23$\times$} & \textbf{4.97} & 1.95$\times$ & 2.94 & \textbf{3.20$\times$} & \textbf{4.89} & 1.51$\times$ & 2.14 & \textbf{2.43$\times$} & \textbf{3.51} \\
GSM8K & 4.35$\times$ & 5.99 & \textbf{6.17$\times$} & \textbf{8.87} & 4.33$\times$ & 5.93 & \textbf{6.27$\times$} & \textbf{8.95} & 3.94$\times$ & 5.08 & \textbf{5.66$\times$} & \textbf{7.38} \\
HumanEval & 4.34$\times$ & 5.97 & \textbf{6.43$\times$} & \textbf{9.18} & 4.00$\times$ & 5.48 & \textbf{6.09$\times$} & \textbf{8.59} & 5.64$\times$ & 7.60 & \textbf{7.88$\times$} & \textbf{10.42} \\
LiveCodeBench & 4.53$\times$ & 6.53 & \textbf{6.44$\times$} & \textbf{9.44} & 4.52$\times$ & 6.61 & \textbf{6.46$\times$} & \textbf{9.53} & 3.77$\times$ & 5.57 & \textbf{5.47$\times$} & \textbf{7.79} \\
MATH-500 & 4.65$\times$ & 6.60 & \textbf{6.60$\times$} & \textbf{9.61} & 4.56$\times$ & 6.46 & \textbf{6.59$\times$} & \textbf{9.54} & 4.01$\times$ & 5.32 & \textbf{5.91$\times$} & \textbf{7.76} \\
MBPP & 4.03$\times$ & 5.56 & \textbf{6.09$\times$} & \textbf{8.72} & 3.83$\times$ & 5.30 & \textbf{5.87$\times$} & \textbf{8.34} & 5.47$\times$ & 7.03 & \textbf{7.55$\times$} & \textbf{9.76} \\
MT-Bench & 2.46$\times$ & 4.05 & \textbf{3.91$\times$} & \textbf{6.27} & 2.30$\times$ & 3.77 & \textbf{3.75$\times$} & \textbf{6.02} & 1.96$\times$ & 3.43 & \textbf{3.10$\times$} & \textbf{5.14} \\
SWE-bench Lite & 2.29$\times$ & 3.07 & \textbf{3.71$\times$} & \textbf{5.20} & 2.10$\times$ & 2.82 & \textbf{3.47$\times$} & \textbf{4.86} & 2.42$\times$ & 3.16 & \textbf{3.80$\times$} & \textbf{4.96} \\
\bottomrule
\end{tabular}
}
\end{table*}

Figure~\ref{fig:speedup-overview} and Table~\ref{tab:benchmark-results} summarize the main benchmark results. Figure~\ref{fig:speedup-overview} summarizes the temperature 0.0 results across datasets and target models, using the best DDTree node budget for each dataset-model pair. Table~\ref{tab:benchmark-results} reports the exact numbers at both temperatures.

DDTree improves every entry in Table~\ref{tab:benchmark-results}, covering all \(10 \times 3 \times 2 = 60\) dataset-model-temperature settings. The gains are consistent across all three target models, across reasoning, code, and general instruction tasks, and at both temperatures.

\subsection{Budget-quality tradeoff}

Figure~\ref{fig:budget-tradeoff} shows a case study on MATH-500 with Qwen3-8B at temperature 0.0. As our DDTree node budget grows, acceptance length increases steadily, and the end-to-end speedup improves until it peaks around budgets of 256 to 512. Pushing the budget to 1024 increases acceptance length further, but the tradeoff is no longer favorable as the additional overhead of verifying more drafted tokens outweighs the gain from the longer accepted prefix. Vanilla DFlash uses a single block of size 16, so the figure also shows that DDTree provides better speedup under the same conceptual budget. This highlights the importance of a front-heavy tree that does not waste budget on low-probability trajectories. Note that the optimal budget can shift across hardware platforms and implementations.

\begin{figure}[H]
\centering
\includegraphics[width=0.6\linewidth]{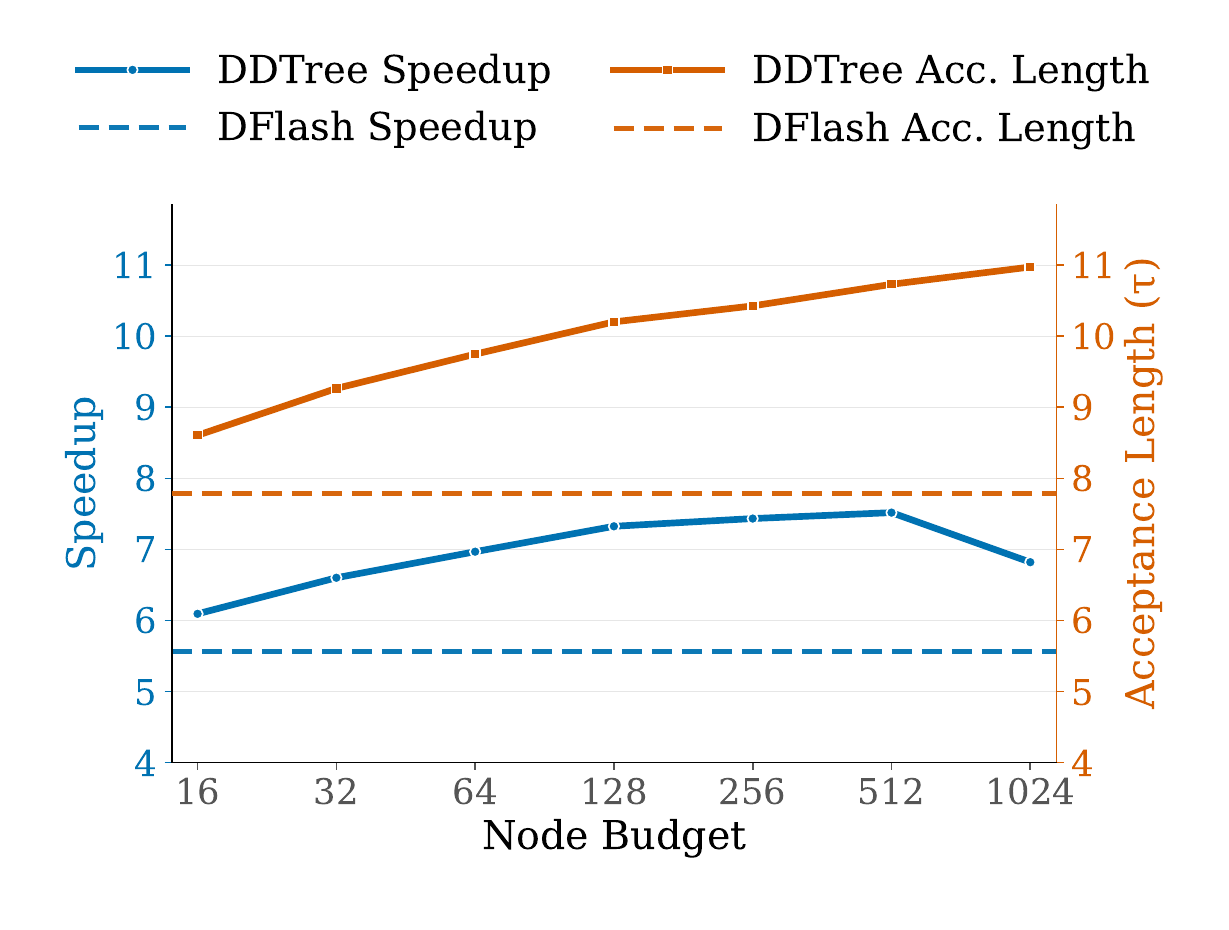}
\caption{Budget tradeoff on MATH-500 with Qwen3-8B at temperature 0.0. Acceptance length increases steadily with the DDTree node budget, while speedup peaks at an intermediate budget once verifier cost becomes dominant.}
\label{fig:budget-tradeoff}
\end{figure}

\subsection{Acceptance length distribution}

Figure~\ref{fig:acceptance-distribution} shows the histogram of acceptance lengths on MATH-500 with Qwen3-8B at temperature 0.0. The plotted DDTree distribution uses the best speedup budget, \(B=512\). Compared with vanilla DFlash, DDTree shifts substantial probability mass toward longer accepted prefixes. With DDTree, it becomes much rarer to observe acceptance lengths below 4, while full-block acceptance at length 16 becomes substantially more common. This shift explains the end-to-end speedup improvement: DDTree makes long accepted prefixes substantially more common, so the verifier needs fewer rounds per generated token.

\begin{figure}[H]
\centering
\includegraphics[width=0.6\linewidth]{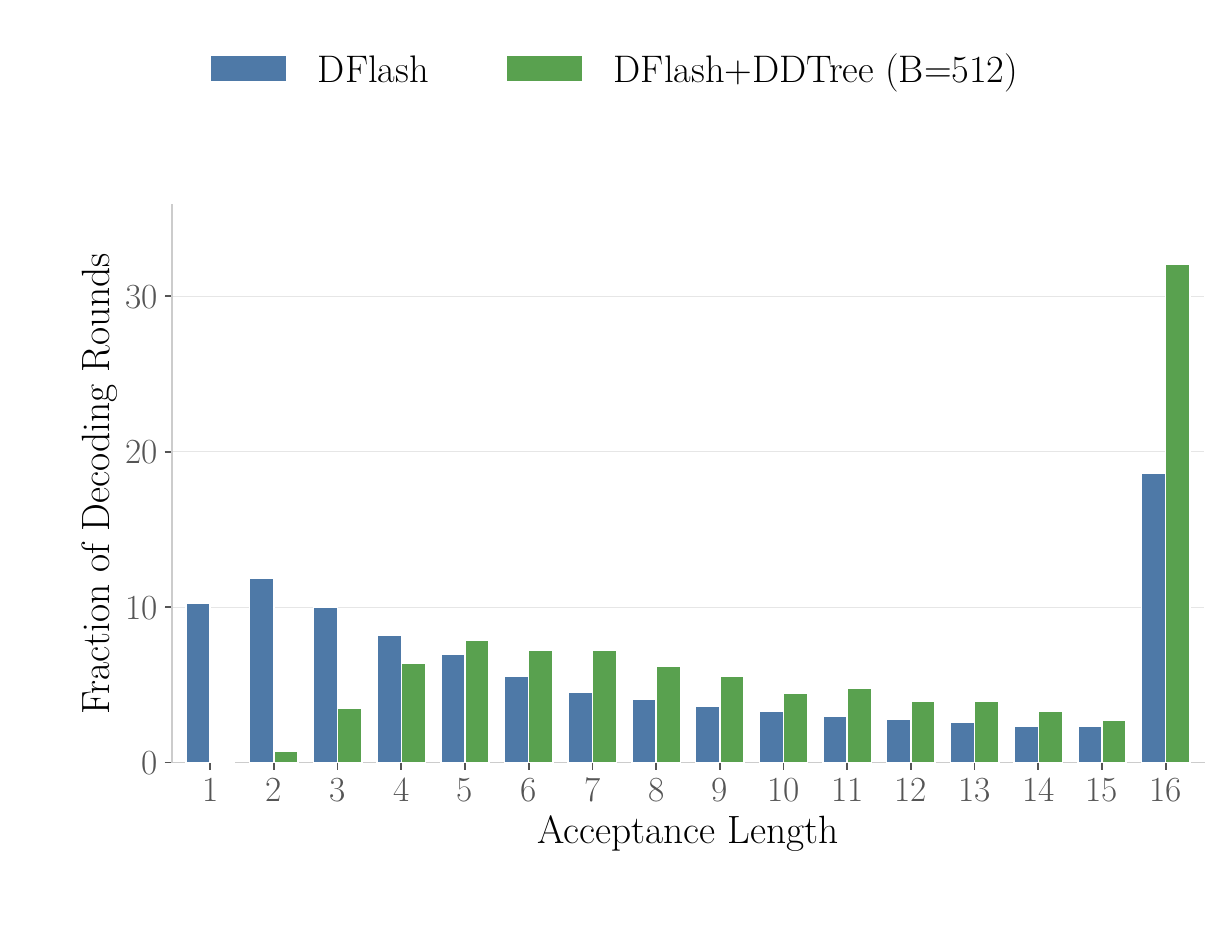}
\caption{Acceptance length distribution on MATH-500 with Qwen3-8B at temperature 0.0. The DDTree histogram uses the best speedup budget, \(B=512\), and shifts mass toward longer accepted prefixes, especially full block acceptances.}
\label{fig:acceptance-distribution}
\end{figure}

\section*{Acknowledgments}
L. R. and Y. R. were supported by the European Union (ERC, SafetyBounds, 101163414). Views and opinions
expressed are however those of the authors only and do not necessarily reflect those of the European Union or
the European Research Council Executive Agency. Neither the European Union nor the granting authority can be
held responsible for them. This research was also partially supported by the Israel Science Foundation (ISF grant
729/21). Y. R. acknowledges additional support from the Career Advancement Fellowship at the Technion. The contribution of the first author is part of a
PhD thesis research conducted at the Technion.

\bibliographystyle{unsrtnat}
\bibliography{references}

\clearpage
\appendix
\renewcommand{\thesection}{\Alph{section}}
\setcounter{section}{0}

\section{Mathematical proofs}
\label{app:proofs}

\begin{proof}[Proof of Proposition \ref{prop:surrogate-additive}]
We expand
\begin{equation}
    \alpha_T(Y_{1:L}) = \sum_{d=1}^{L} \mathbf{1}\{\alpha_T(Y_{1:L}) \ge d\}.
\end{equation}
For a fixed depth \(d\), the event \(\{\alpha_T(Y_{1:L}) \ge d\}\) holds if and only if the sampled depth-\(d\) prefix \(Y_{1:d}\) is one of the depth-\(d\) nodes in \(T\). These depth-\(d\) events are disjoint because \(Y_{1:d}\) can equal only one sequence. Therefore
\begin{equation}
    \Pr[\alpha_T(Y_{1:L}) \ge d]
    = \sum_{u \in T: |u| = d} \Pr[Y_{1:d} = u]
    = \sum_{u \in T: |u| = d} q(u \mid c,b).
\end{equation}
Summing over \(d\) gives the claim.
\end{proof}

\begin{proof}[Proof of Proposition \ref{prop:top-b-tree}]
Let \(v\) be any strict descendant of \(u\). Then
\begin{equation}
    q(v \mid c,b) = q(u \mid c,b) \prod_{i=|u|+1}^{|v|} q_i(v_i \mid c,b) < q(u \mid c,b).
\end{equation}
Therefore, every strict ancestor has a strictly larger probability than any of its descendants. It follows that, in any nonincreasing ordering of prefixes, every ancestor must appear before its descendants, which implies that \(T_B\) is prefix-closed and hence valid.

For optimality, Proposition~\ref{prop:surrogate-additive} shows that the surrogate objective for a tree \(T\) is
\[
\sum_{u \in T} q(u \mid c,b).
\]
This objective is additive over nodes, and every term is nonnegative. Therefore, among all sets of at most \(B\) prefixes, the maximum is attained by taking the top-\(B\) probabilities \(q(u \mid c,b)\), namely the prefixes in \(T_B\). Since \(T_B\) is valid, it is optimal among all valid draft trees with \(|T| \le B\). If ties occur between unrelated prefixes, other optimal trees may also exist.
\end{proof}

\begin{proof}[Proof of Lemma \ref{lem:top-k-reduction}]
Order all nonempty prefixes by nonincreasing \(q(u\mid c,b)\), breaking ties by placing prefixes in \(\mathcal{S}_K\) before prefixes outside \(\mathcal{S}_K\). By Proposition~\ref{prop:top-b-tree}, the first \(B\) prefixes in this order form an optimal valid draft tree. It therefore suffices to show that all of these first \(B\) prefixes lie in \(\mathcal{S}_K\).

If \(B\geq|\vocab|\), then \(K=\min(B,|\vocab|)=|\vocab|\), so every prefix already lies in \(\mathcal{S}_K\). Thus only the case \(B < |\vocab|\) remains, in which case \(K=B\).

Assume by contradiction that some prefix \(u=(u_1,\dots,u_d)\notin \mathcal{S}_K\) appears among the first \(B\) prefixes. Let
\[
I=\{j\in\{1,\dots,d\}:\text{the rank of }u_j\text{ at depth }j\text{ exceeds }B\},
\]
pick any \(i\in I\), and let \(\tilde u\) be obtained from \(u\) by replacing \(u_j\) with \(v_j^{(1)}\) for every \(j\in I\). For each \(k\in\{1,\dots,B\}\), let \(u^{(k)}\) be obtained from \(\tilde u\) by replacing its \(i\)-th coordinate with \(v_i^{(k)}\).

Then \(u^{(1)},\dots,u^{(B)}\) are \(B\) distinct prefixes in \(\mathcal{S}_K\), and for every \(k\),
\[
q(u^{(k)}\mid c,b)\ge q(u\mid c,b),
\]
since each replaced coordinate is changed to a token of rank at most \(B\). By the tie-breaking rule, each \(u^{(k)}\) appears before \(u\) in the ordering. Hence at least \(B\) prefixes appear before \(u\), contradicting that \(u\) is among the first \(B\).

Therefore no prefix outside \(\mathcal{S}_K\) can appear among the first \(B\) prefixes. Hence all first \(B\) prefixes lie in \(\mathcal{S}_K\), and so some optimal valid draft tree is contained in \(\mathcal{S}_K\).
\end{proof}

\begin{proof}[Proof of Proposition \ref{prop:alg_optimal_tree}]
Each \(\rho \in \mathcal{S}_K\) corresponds to the prefix \((v_1^{(\rho_1)},\dots,v_d^{(\rho_d)})\), and its score is
\[
\log q(\rho) = \sum_{i=1}^d \log q_i^{(\rho_i)}.
\]
By Lemma~\ref{lem:top-k-reduction}, some optimal valid draft tree is contained in \(\mathcal{S}_K\). It is therefore enough to show that Algorithm~\ref{alg:tree-build} returns the \(B\) highest-scoring elements of \(\mathcal{S}_K\), in nonincreasing order of \(q(\rho)\).

For every \(\rho=(\rho_1,\dots,\rho_d) \in \mathcal{S}_K\) other than \((1)\), define its predecessor by
\[
\operatorname{pred}(\rho)=
\begin{cases}
(\rho_1,\dots,\rho_{d-1},\rho_d-1), & \rho_d > 1, \\
(\rho_1,\dots,\rho_{d-1}), & \rho_d=1,\ d>1.
\end{cases}
\]
This predecessor is unique. If \(\rho_d>1\), then popping \(\operatorname{pred}(\rho)\) generates \(\rho\) as the next sibling. If \(\rho_d=1\) and \(d>1\), then popping \(\operatorname{pred}(\rho)\) generates \(\rho\) as the first child.

We next show that the heap pops tuples in nonincreasing order of \(q(\rho)\). The first pop is \((1)\), since it is the unique initial heap element. Now consider any later iteration. At that point, the heap contains exactly the unpopped tuples whose predecessor has already been popped. Fix any unpopped tuple \(\rho\), and let \(\bar{\rho}\) be the first tuple on the predecessor chain from \(\rho\) back toward \((1)\) whose predecessor has already been popped. Then \(\bar{\rho}\) is in the heap. Moreover,
\[
q(\operatorname{pred}(\rho)) \ge q(\rho)
\]
for every noninitial tuple \(\rho\): in the sibling case this replaces \(q_d^{(\rho_d)}\) by the larger quantity \(q_d^{(\rho_d-1)}\), and in the child case it removes the factor \(q_d^{(1)} < 1\). Applying this inequality repeatedly along the chain from \(\rho\) to \(\bar{\rho}\) gives \(q(\bar{\rho}) \ge q(\rho)\). Therefore the maximum-probability unpopped tuple is always present in the heap, and each pop removes a highest-probability remaining tuple. Repeating this argument proves that the pop sequence is nonincreasing in \(q(\rho)\).

After \(B\) pops, the returned set \(T\) therefore consists of the top-\(B\) prefixes. It remains to show that \(T\) is valid. Let \(\rho=(\rho_1,\dots,\rho_d) \in T\) with \(d \ge 2\). Repeatedly applying \(\operatorname{pred}\) to \(\rho\) decreases the last coordinate until reaching \((\rho_1,\dots,\rho_{d-1},1)\), and one more predecessor step yields the parent tuple \((\rho_1,\dots,\rho_{d-1})\). Since each tuple is generated only after its predecessor is popped, every tuple on this chain was popped earlier and therefore belongs to \(T\). Hence the parent of every returned nonroot prefix is also returned, so \(T\) is prefix-closed.

After \(B\) pops, the returned set \(T\) is exactly the top-\(B\) elements of \(\mathcal{S}_K\). By Proposition~\ref{prop:top-b-tree}, \(T\) is a valid draft tree and is optimal among trees contained in \(\mathcal{S}_K\). By Lemma~\ref{lem:top-k-reduction}, some global optimum is contained in \(\mathcal{S}_K\). Hence \(T\) is optimal for \eqref{eq:surrogate-objective}.

Finally, Proposition~\ref{prop:surrogate-additive} implies that, within \(\mathcal{S}_K\), the surrogate objective is maximized by taking the highest-probability prefixes. Since Algorithm~\ref{alg:tree-build} returns exactly those prefixes, and some global optimum lies in \(\mathcal{S}_K\), the returned tree is optimal for \eqref{eq:surrogate-objective} under node budget \(B\).
\end{proof}

\section{Benchmark details}
\label{app:benchmark-details}
All runs use block size 16, DDTree node budgets \(\{16, 32, 64, 128, 256, 512, 1024\}\), temperatures 0.0 and 1.0, a maximum of 2048 new tokens, and bfloat16 inference. We run the benchmark on 8 H200 GPUs and shard the evaluation set across workers. Before timing the benchmark loop, we run a short warmup prompt through the autoregressive baseline, vanilla DFlash, and each DDTree budget so that one time setup costs are excluded from the reported measurements.

For DDTree, the target model uses standard PyTorch scaled dot product attention, because FlashAttention-2~\citep{dao2024flashattention} does not support the required tree attention pattern. The DFlash drafter itself still uses FlashAttention-2. For fairness, for the autoregressive baseline and vanilla DFlash, we evaluate the target model with both standard PyTorch scaled dot product attention and FlashAttention-2, and report the faster result, which can only improve these baselines relative to DDTree.

Table~\ref{tab:dataset-sample-counts} lists the number of evaluated examples for each dataset. We follow the original DFlash benchmark setup for these sample counts.

\begin{table}[h]
\centering
\small
\caption{Number of evaluated examples per dataset in the benchmark suite.}
\label{tab:dataset-sample-counts}
\begin{tabular}{lr}
\toprule
\textbf{Dataset} & \textbf{Examples} \\
\midrule
AIME 2024 & 30 \\
AIME 2025 & 30 \\
Alpaca & 128 \\
GSM8K & 128 \\
HumanEval & 164 \\
LiveCodeBench & 128 \\
MATH-500 & 128 \\
MBPP & 128 \\
MT-Bench & 80 \\
SWE-bench Lite & 128 \\
\bottomrule
\end{tabular}
\end{table}

\end{document}